\documentclass[conference]{IEEEtran}
\IEEEoverridecommandlockouts

\usepackage{cite}
\usepackage{amsmath,amssymb,amsfonts}
\usepackage{amsthm}
\usepackage{algorithmic}
\usepackage{graphicx}
\usepackage{textcomp}
\usepackage{xcolor}
\newtheorem{theorem}{Theorem}
\usepackage{graphicx}
\usepackage{booktabs}
\usepackage[caption=false,font=footnotesize]{subfig} 
\usepackage{authblk} 
\usepackage{hyperref} 
\usepackage[capitalize]{cleveref}
\usepackage{orcidlink}

\title{PC-UNet: An Enforcing Poisson Statistics U-Net for Positron Emission Tomography Denoising\\
}
\author[1,2*]{Yang Shi\orcidlink{0009-0009-3928-7495}\thanks{* These authors contributed equally to this work: sudo.shiyang@gmail.com, ethanwangjc@163.com, lu.liangsi.cn@gmail.com}}
\author[3*]{Jingchao Wang\orcidlink{0000-0002-0099-539X}}
\author[4*]{Liangsi Lu\orcidlink{0009-0006-2839-3901}}
\author[5]{Mingxuan Huang\orcidlink{0009-0005-3428-2119}}
\author[1]{Ruixin He\orcidlink{0009-0000-6137-4901}}
\author[4]{Yifeng Xie\orcidlink{0009-0008-8333-9419}}
\author[6]{\\Hanqian Liu}
\author[1]{Minzhe Guo\orcidlink{0009-0001-2511-4763}}
\author[1]{Yangyang Liang}
\author[7]{Weipeng Zhang\orcidlink{0009-0001-5110-565X}}
\author[2\dag]{Zimeng Li\thanks{\dag~Corresponding authors: li\_zimeng@szpu.edu.cn, xuhangc@hzu.edu.cn.}}
\author[8]{Xuhang Chen\orcidlink{0000-0001-6000-3914}\thanks{This work was supported in part by Shenzhen Medical Research Fund (Grant No. A2503006), in part by the National Natural Science Foundation of China (Grant No. 62501412), in part by Shenzhen Polytechnic University Research Fund (Grant No. 6025310023K) and in part by Guangdong Basic and Applied Basic Research Foundation (Grant No. 2024A1515140010).}}

\affil[1]{School of Computer Science and Technology, Guangdong University of Technology, Guangzhou, China}
\affil[2]{School of Electronic and Communication Engineering, Shenzhen Polytechnic University, Shenzhen, China}
\affil[3]{School of Computer Science, Peking University, Beijing, China}
\affil[4]{School of Mathematics and Statistics, Guangdong University of Technology, Guangzhou, China}
\affil[5]{Zhongshan School of Medicine, Sun Yat-sen University, Guangzhou, China}
\affil[6]{School of Mathematics (Zhuhai), Sun Yat-sen University, Guangzhou, China}
\affil[7]{Future Technology Institute, South China University of Technology, Guangzhou, China}
\affil[8]{School of Computer Science and Engineering, Huizhou University, Huizhou, China}

\begin{document}

\maketitle
\begin{abstract}
Positron Emission Tomography (PET) is crucial in medicine, but its clinical use is limited due to high signal-to-noise ratio doses increasing radiation exposure. Lowering doses increases Poisson noise, which current denoising methods fail to handle, causing distortions and artifacts. We propose a Poisson Consistent U-Net (PC-UNet) model with a new Poisson Variance and Mean Consistency Loss (PVMC-Loss) that incorporates physical data to improve image fidelity. PVMC-Loss is statistically unbiased in variance and gradient adaptation, acting as a Generalized Method of Moments implementation, offering robustness to minor data mismatches. Tests on PET datasets show PC-UNet improves physical consistency and image fidelity, proving its ability to integrate physical information effectively.
\end{abstract}


\section{Introduction}

With the rapid advancement of deep learning~\cite{emnlp/XieZCHC23,mm/XieZCCH24,wang2024novel,wang2024beyond}, medical imaging encompasses not only anatomical depiction but also embrace functional and molecular interrogation of disease. Methods such as X-ray and Computed Tomography focus on morphology~\cite{hussain2022modern}, while Magnetic Resonance Imaging (MRI) is superior in differentiating soft tissues~\cite{bottomley1982nmr}. Positron Emission Tomography (PET) provides insight into cellular metabolism, aiding in early cancer detection, accurate staging, and monitoring therapy response~\cite{pain2022deep}.

PET, despite its clinical utility, remains the noisiest imaging modality due to Poisson statistics affecting photon detection: lower doses lead to fewer counts and more noise. Initially, simple CNNs were used, later succeeded by U-Nets, GANs, and diffusion models, trained on L1/L2 losses~\cite{seyyedi2024deep,shi2019novel}. The required standard dose for diagnostic images raises radiation exposure~\cite{boellaard2015fdg}, prompting dose reduction to lower patient risk~\cite{akita202518f}. However, fewer photons mean noisier images, reducing lesion detectability~\cite{yan2016method,shepp2007maximum}. Improving image quality under low-dose constraints is a key challenge~\cite{hu2023comparative}. Many studies~\cite{pain2022deep} first reconstruct a noisy image, then enhance it with deep networks.

\begin{figure}[t]
    \centering
    \begin{minipage}{0.157\textwidth}
        \centering
        \includegraphics[width=\linewidth]{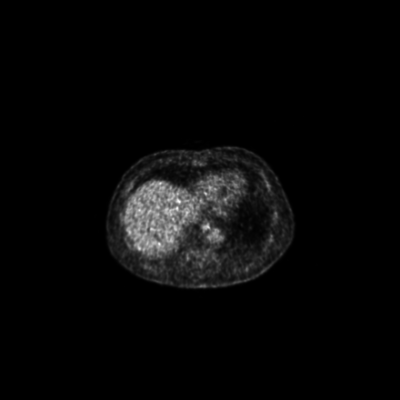}
        \par(a)
    \end{minipage} 
    \begin{minipage}{0.157\textwidth}
        \centering
        \includegraphics[width=\linewidth]{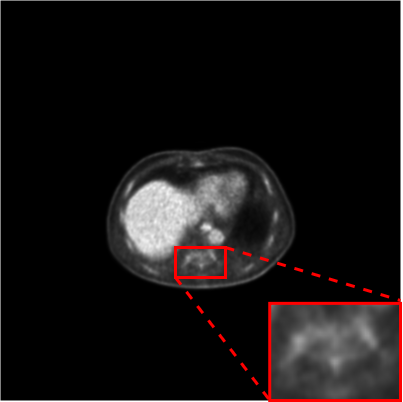}
        \par(b)
    \end{minipage} 
    \begin{minipage}{0.157\textwidth}
        \centering
        \includegraphics[width=\linewidth]{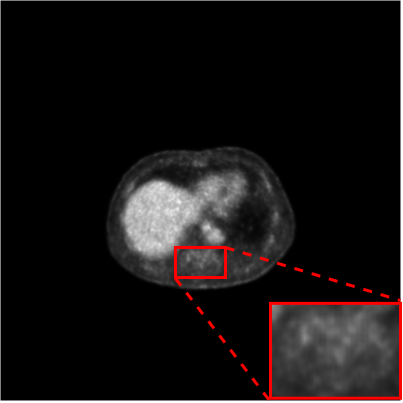}
        \par(c)
    \end{minipage}
    \caption{(a) is a low-dose PET. For the small number of photons, the signal of (a) is completely overwhelmed by noise. (b) is a full-dose PET, we consider it as a clean image. (c) is obtained by denoising (a) through U-Net. The red rectangular area marked out is a region with relatively low photon count. This characteristic is manifested as lower brightness but clear structure in (b), while (c) appears as blurred structure and severe noise artifacts.
    }
    \label{fig:intro}
\end{figure}


Despite progress, \cref{fig:intro} shows limitations. Without physical constraints, networks overly smooth strong noise in bright areas, erasing details, and fail to address noise in dark areas, causing artifacts. In low-dose conditions, photon events follow a Poisson distribution, where noise variance is proportional to signal mean. Strong signal areas have intense noise, while weak signal areas have less. L1 or L2 loss functions treat all pixels equally, reducing errors uniformly \cite{liu2024cross}.
 

We propose \textbf{P}oisson \textbf{C}onsistent \textbf{U}-\textbf{Net} (PC-UNet), a framework that improves denoising by incorporating physical principles into its optimization. PC-UNet features \textbf{P}oisson \textbf{V}ariance and \textbf{M}ean \textbf{C}onsistency \textbf{L}oss (PVMC-Loss), which constrains the model to adhere to the imaging process's physical principles. It enforces the ratio of local noise variance to the mean of the denoised signal, aligning the model with Poisson statistics and enhancing consistency and robustness. 

A theoretical analysis confirms the effectiveness of PVMC-Loss, with proofs of its asymptotic unbiasedness and adaptive gradients. These properties prevent systematic value distortion and prioritize challenging low-signal areas, respectively. By interpreting PVMC-Loss within the Generalized Method of Moments (GMM) framework \cite{hansen1982large}, we link our method to robust statistical principles, explaining its accuracy improvements.

The main contributions of this paper can be summarized as follows:
\begin{itemize}
\item We propose PC-UNet, a novel framework that incorporates physical constraints into the training process to overcome the inherent limitations of conventional U-Net.
\item To ensure that the network’s output obeys the physical Poisson statistics of low-dose PET, we design PVMC-Loss, a loss function that explicitly enforces the ratio between residual-noise variance and local signal mean.
\item We establish a theoretical foundation for the proposed method, prove its effectiveness, and demonstrate its connection to the GMM, thereby providing statistical justification for the improved quantitative accuracy.
\end{itemize}

\section{Method}

The loss function of PC-UNet is composed of L1 loss and PVMC-Loss. In this section, we derive and construct our proposed PVMC-Loss from the underlying physical principles of PET imaging and provide a complete theoretical property analysis for it. Moreover, we build our proposed PC-UNet. The framework of PC-UNet is shown in \cref{fig:method}.

\subsection{PVMC-Loss}

The PVMC-Loss is derived from the PET count statistics and linear reconstruction theory, and the formal definition of this loss function is given. The physical basis of PET imaging is the photon counting process, which inherently follows a Poisson distribution. Specifically, the detector count $N_j$ for each Line Of Response (LOR) can be modeled as an independent Poisson random variable with an expectation, $N_{j}\sim\text{Poisson}(\lambda_{j})$, equal to the true photon intensity $\lambda_j$:

\begin{equation}
Var(N_{j})=\mathbb{E}(N_{j})=\lambda_{j},
\end{equation}
where $Var(\cdot)$ is defined as the variance and $\mathbb{E}(\cdot)$ is defined as mathematical expectation.

However, clinical PET images are not direct representations of raw counts, but undergo complex correction and reconstruction processes. Given sufficient iteration or filtered back projection (FBP), the value $\hat{y}_i$ of voxel $i$ in the reconstructed image can be approximated as a weighted linear combination of all LOR counts:
\begin{equation}
\hat{y}_{i}= \sum_{j} w_{ij}\,c_{j}\,N_{j},
\end{equation}
where $\hat{y}_i$ is the value of the voxel $i$, $w_{ij}$ is the reconstruction weight defined by the system matrix, and $c_{j}$ is the known constant coefficient used to correct for scattering, attenuation, and detector sensitivity. 

Based on this linear reconstruction model, the variance-mean relationship of the reconstructed image voxels can be derived. We compute the expectation of the reconstructed voxel value $\hat{y}_i$. By the expected linearity property, we have:
\begin{equation}
\begin{aligned}
\mathbb{E}(\hat{y}_{i}) &= \mathbb{E}\left(\sum_{j} w_{ij} c_j N_j\right) = \sum_{j} w_{ij} c_j \mathbb{E}(N_j)\\
&=\sum_{j} w_{ij} c_j \lambda_j.
\end{aligned}
\end{equation}

Similarly, the variance of $\hat{y}_i$ can be derived as follows:
\begin{equation}
\begin{aligned}
Var(\hat{y}_{i}) &= Var\left(\sum_{j} w_{ij} c_j N_j\right) = \sum_{j} (w_{ij} c_j)^2 Var(N_j)\\
&=\sum_{j} (w_{ij} c_j)^2 \lambda_j.
\end{aligned}
\end{equation}

To obtain a universal relation, we assume that the background activity is approximately constant in a locally uniform neighborhood, $\lambda_j \approx \lambda$. Under this condition, the above expectation and variance can be simplified as follows:

\begin{equation}
\begin{aligned}
\mathbb{E}(\hat{y}_{i}) \approx \lambda \sum_{j} w_{ij} c_j,
\end{aligned}
\end{equation}
\begin{equation}
\begin{aligned}
\quad Var(\hat{y}_{i}) \approx \lambda \sum_{j} w_{ij}^{2} c_{j}^{2}.
\end{aligned}
\end{equation}

\begin{figure}[t]
    \centering
    \includegraphics[width=0.95\linewidth]{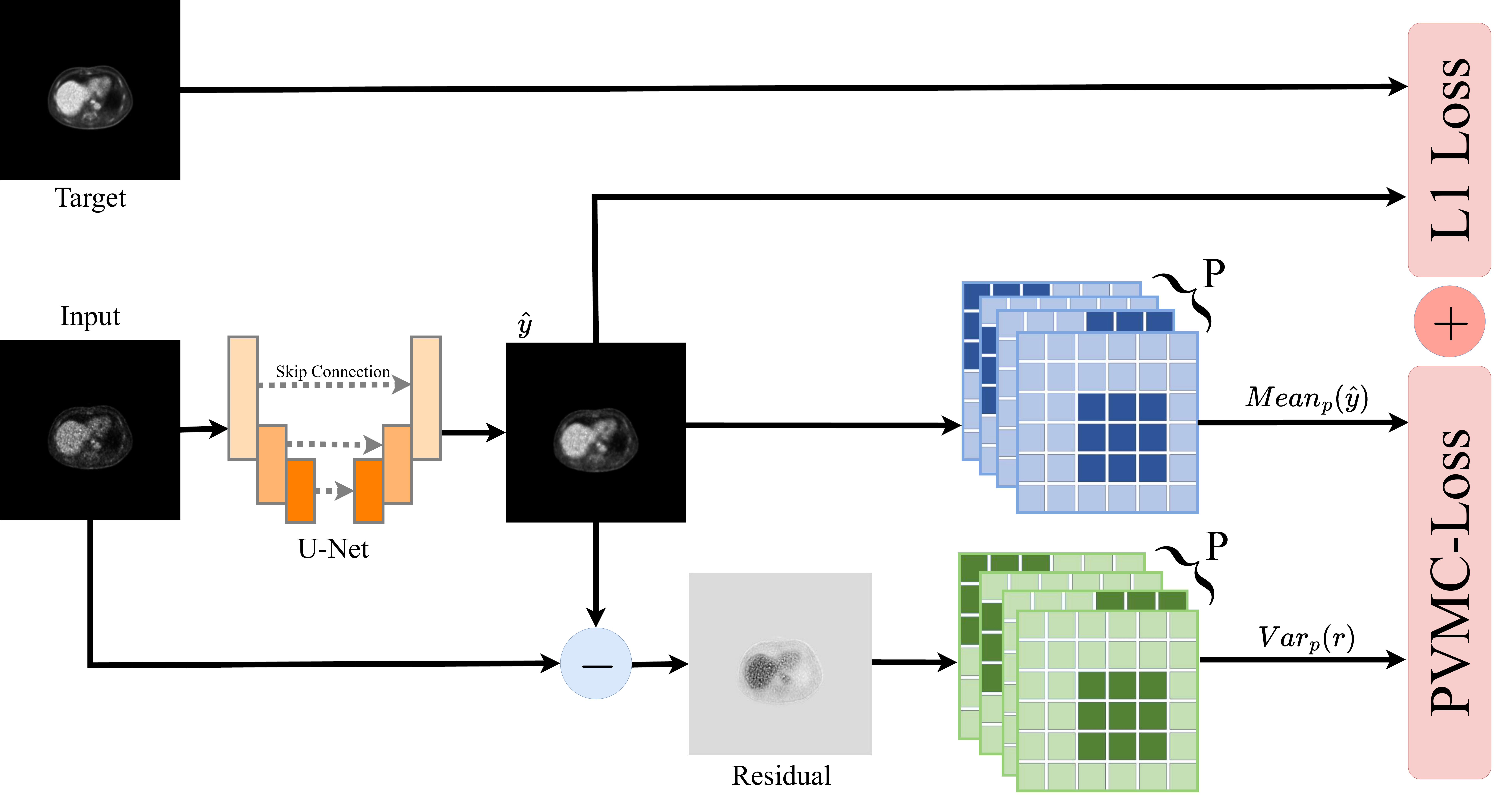}
    \caption{
Framework of PC-UNet. We use the denoised image and the patch of the residual image, calculating the mean and variance to obtain our proposed PVMC-Loss.
    }
    \label{fig:method}
\end{figure}

By calculating the ratio of the variance to the expectation, we derive a key physical parameter:

\begin{equation}
\begin{aligned}
k = \frac{Var(\hat{y}_{i})}{\mathbb{E}(\hat{y}_{i})} = \frac{\lambda \sum_j w_{ij}^{2}c_j^{2}}{\lambda \sum_j w_{ij} c_j} = \frac{\sum_j w_{ij}^{2}c_j^{2}}{\sum_j w_{ij} c_j} > 0,
\end{aligned}
\end{equation}
where $k$ is defined as the poisson slope. The activity term $\lambda$ in this ratio is completely eliminated, so that $k$ only depends on the geometry of the scanner, the correction factor and the filter kernel used in the reconstruction, and can be regarded as a global constant under a fixed scanning protocol. This establishes $k$ as a physical constant for a given scanning protocol. However, for practical implementation where precise calibration might be unavailable, we propose and validate a flexible strategy of treating $k$ as a learnable parameter co-optimized with the network.

Based on this physical relationship, we construct the constraint objective for the denoising task. Denoising network is defined as $f_{\theta}(\cdot)$, input a low-dose image $x$, and output a denoising estimate $\hat{y}$. The noise Residual is denoted as $r := x - \hat{y}$. If the network can be perfectly denoised, $\hat{y}$ approximates the true noise-free signal $y$, then the statistical characteristics of the residual $r$ should be consistent with the noise in the original imaging process, that is, it satisfies:
\begin{equation}
\begin{aligned}
Var(r) \approx k \cdot y.
\end{aligned}
\end{equation}

By approximating $y$ with $\hat{y}$ during training and enforcing this constraint on randomly sampled patches $p$, we can obtain the final form of the local constraint:

\begin{equation}
\begin{aligned}
Var_{p}(r) \approx k\,{Mean}_{p}(\hat{y}),
\label{equ:k}
\end{aligned}
\end{equation}
where $Var_{p}(\cdot)$ is defined as the unbiased sample variance calculated for the local region $p$ of the image and ${Mean}_{p}(\cdot)$ is defined as the sample mean calculated over the local patch $p$ of the image. In order to robustly estimate the local mean and variance in \cref{equ:k} in practical calculations, we adopt an unbiased random sampling strategy. A continuous voxel block of size $(s_x, s_y)$ is intercepted by randomly selecting the starting coordinates $(x_0, y_0)$ on a 3D Cartesian grid, and the set of voxel indices of this block is defined as $p$. For any tensor $z$, its local mean and unbiased sample variance over patch $p$ are defined as follows:
\begin{equation}
\begin{aligned}
{Mean}_{p}(z) = \frac{1}{S}\sum_{k\in p}z_{k},
\end{aligned}
\end{equation}
\begin{equation}
\begin{aligned}
Var_{p}(z) = \frac{1}{S-1}\sum_{k\in p} \bigl(z_{k}-{Mean}_{p}(z)\bigr)^{2},
\end{aligned}
\end{equation}
where $S=s_{x}s_{y}s_{z}$ is the total number of voxels in the patch and we define .

According to the above, our proposed PVMC-Loss can be formally defined as. Given $P$ sampled patches within a batch, the loss function is as follows:
\begin{equation}
\begin{aligned}
\mathcal{L}_{\text{PVMC}}
   =\frac{1}{P}\sum_{p=1}^{P}
   \Bigl|
      \pi_{p}-1
   \Bigr|,
\end{aligned}
\end{equation}
where $|\cdot|$ represents the absolute value function and $\pi_{p}$ is defined as:
\begin{equation}
\begin{aligned}
\pi_{p}
   =\frac{Var_{p}(r)}
         {k\,{Mean}_{p}(\hat{y})+\varepsilon},
\end{aligned}
\end{equation}
where $\varepsilon$ is a minimal positive constant used to prevent the denominator from being zero and to ensure numerical stability, especially in low-count patches where the mean value may approach zero. When $\mathcal{L}_{\text{PVMC}}\to 0$, \cref{equ:k} holds approximately within all sampled patches, thus ensuring that the network output maintains the correct physical scaling relationship in a statistical sense.

Our derivation of the constant $k$ relies on the assumption of a locally uniform activity distribution ($\lambda_{j}\approx\lambda$). While this holds true for background and larger, homogeneous tissue regions, it may be less accurate at sharp boundaries like tumor edges. Future work could explore adaptive methods where $k$ might vary spatially to account for such high-contrast interfaces. In this work, to determine the value of Poisson slope $k$ for a particular scanning protocol, we treat $k$ as a learnable scalar that is co-optimized with the network weights at training time, and verify in the experimental part that it is highly consistent with the offline calibration results, demonstrating the effectiveness and convenience of the method. 

\subsection{Theoretical Analysis of PVMC-Loss}
This section proves a series of properties possessed by our proposed PVMC-Loss.

\subsubsection{Asymptotic Unbiasedness with Bounded Bias}
In our theoretical analysis, we start from two fundamental premises. Firstly, under the standard PET imaging model, the noisy observation value $x$ is an unbiased estimate of the true signal $y$, $\mathbb{E}(x)=y$ \cite{shepp2007maximum,zaidi2007scatter}. Secondly, we adopt a core assumption from the denoising theory, for a network that has achieved convergence in training $\mathcal{L}_{\text{total}}\to0$, the output $\hat{y}$ is asymptotically weakly correlated with the residual $r$. For any local image block $p$, the covariance satisfies $Cov_p(r,\hat y) \to 0$ \cite{lehtinen2018noise2noise,batson2019noise2self}. This assumption stems from the idea that an ideal denoiser should be able to effectively separate the signal from random noise, and it has become a widely accepted theoretical foundation for analyzing the behavior of network cascades. Although this is an idealized condition and there may be weak residual correlations in actual networks with limited capacity, we believe that PVMC-loss, through its unique physical constraint, namely forcing the variance of residuals to be coupled with the mean of the signal, can actively regularize the network and make its behavior closer to this ideal state compared to unconstrained models. Based on these premises, we can deduce that our method possesses a certain property; the expected bias of the model, $\mathbb{E}(\hat{y}-y)$, is not a random distribution but is proportional to the local variance of the denoised signal, $Var(\hat{y})$. This explains the inherent and controllable smoothing effect of deep learning methods. We provide a formal description and proof of this in \cref{theo:1}.

\begin{theorem}\label{theo:1}
When $\mathcal{L}_{\mathrm{PVMC}} \to 0$, the expectation of the network output $\hat{y}$ satisfies:

\begin{equation}
\begin{aligned}
\mathbb{E}(\hat{y}) = y - \frac{1}{k}\mathbb{E}_{p \sim D}(Var_{p}(\hat{y})),
\end{aligned}
\end{equation}
where $\mathbb{E}_{p \sim D}(\cdot)$ is defined as first calculating the variance within each patch and then taking the expectation of the variance values of all patches globally.
\end{theorem}

\begin{proof}
According to the definition of PVMC-Loss, the necessary and sufficient condition for the loss function $\mathcal{L}_{\mathrm{PVMC}}(\hat{y}) \to 0$ is that the core ratio $\pi_{p}(\hat{y})$ for all sampled image blocks $p$ approaches 1. Ignoring the minor term $\epsilon$, this condition is equivalent to $Var_{p}(r) \approx k \cdot Mean_{p}(\hat{y})$.

According to the properties of covariance, $Cov_{p}(x, \hat{y})$ can be derived as:
\begin{equation}
\begin{aligned}
Cov_{p}(x, \hat{y})& = Cov_{p}(r+\hat{y}, \hat{y}) = Cov_{p}(r, \hat{y}) + Var_{p}(\hat{y})\\
&\approx Var_{p}(\hat{y}).
\nonumber
\end{aligned}
\end{equation}

We decompose the sample variance of the residual $r = x - \hat{y}$:
\begin{equation}
\begin{aligned}
Var_{p}(r) &= Var_{p}(x) + Var_{p}(\hat{y}) - 2Cov_{p}(x, \hat{y})\\
&\approx Var_{p}(x) - Var_{p}(\hat{y}).
\nonumber
\end{aligned}
\end{equation}

According to \cref{equ:k}, $k \cdot Mean_{p}(\hat{y})$ can be derived as:
\begin{equation}
\begin{aligned}
k \cdot Mean_{p}(\hat{y}) &\approx Var_{p}(r) \approx Var_{p}(x) - Var_{p}(\hat{y}) \\
&\approx k \cdot Mean_{p}(x) - Var_{p}(\hat{y}).
\nonumber
\end{aligned}
\end{equation}

If the image block $p$ is independently and identically distributed, randomly uniformly sampled at the voxel level, then taking the global expectation of the above formula results in:
\begin{equation}
\begin{aligned}
k \cdot \mathbb{E}(\hat{y}) &\approx k \cdot \mathbb{E}(x) - \mathbb{E}_{p \sim D}(Var_{p}(\hat{y}))\\
&=k \cdot y - \mathbb{E}_{p \sim D}(Var_{p}(\hat{y})).
\nonumber
\end{aligned}
\end{equation}
\begin{equation}
\begin{aligned}
\mathbb{E}(\hat{y}) = y - \frac{1}{k}\mathbb{E}_{p \sim D}(Var_{p}(\hat{y})).
\nonumber
\end{aligned}
\end{equation}

\end{proof}

This theorem reveals that the expectation of the network output $\mathbb{E}(\hat{y})$ does not perfectly match the true signal $y$, but is offset by a bias term, $\frac{1}{k}\mathbb{E}_{p\sim D}(Var_{p}(\hat{y}))$, which is proportional to the average local variance of the denoised output itself. This term represents the smoothing effect of the network; therefore, achieving near-unbiased estimation requires this smoothing-induced bias to be minimal.

\subsubsection{Gradient Structure and Adaptive Learning}
\begin{theorem}
For any voxel $\hat{y}_{\textbf{k}}$, where $\textbf{k} \in p$, the exact form of the single block loss $\mathcal{L}_{p}=|\pi_{p}-1|$ on its gradient is given by: 
\begin{equation}
\begin{aligned}
\frac{\partial\mathcal{L}_{p}}{\partial\hat{y}_{\textbf{k}}} = \mathrm{sgn}(\pi_{p}-1) \cdot \frac{\frac{-2(r_{\textbf{k}}-\overline{r}_{p})}{S-1}(k\overline{y}_{p}+\epsilon) - k\frac{Var_{p}(r)}{S}}{(k\overline{y}_{p}+\epsilon)^2},
\end{aligned}
\end{equation}
where $\mathrm{sgn}(\cdot)$ is defined as the sign function, and $\overline{x}$ is defined as the sample mean of a scalar value $x$.

\end{theorem}

\begin{proof}
According to the chain rule:
\begin{equation}
\begin{aligned}
\frac{\partial\mathcal{L}_{p}}{\partial\hat{y}_{\textbf{k}}} = \mathrm{sgn}(\pi_{p}-1)\frac{\partial\pi_{p}}{\partial\hat{y}_{\textbf{k}}}.
\nonumber
\end{aligned}
\end{equation}

Let $\pi_{p} = N/D$, where $N = Var_{p}(r) = \frac{1}{S-1}\sum_{i \in p}(r_{i}-\overline{r}_{p})^{2}$ and $D = k\overline{y}_{p} + \epsilon$. Take the partial derivative of $D$ and $N$ as follows:
\begin{equation}
\begin{aligned}
\frac{\partial D}{\partial\hat{y}_{\textbf{k}}} = \frac{k}{S},
\nonumber
\end{aligned}
\end{equation}
\begin{equation}
\begin{aligned}
\frac{\partial N}{\partial\hat{y}_{\textbf{k}}} = \frac{2}{S-1}\sum_{i \in p}(r_{i}-\overline{r}_{p})(\frac{\partial r_{i}}{\partial\hat{y}_{\textbf{k}}}-\frac{\partial\overline{r}_{p}}{\partial\hat{y}_{\textbf{k}}}),
\nonumber
\end{aligned}
\end{equation}
where $\frac{\partial r_{i}}{\partial\hat{y}_{\textbf{k}}} = -\delta_{i\textbf{k}}$, $\delta_{i\textbf{k}}$ is defined as the Kronecker symbol and $\frac{\partial\overline{r}_{p}}{\partial\hat{y}_{\textbf{k}}} = -\frac{1}{S}$, so the partial derivative of $N$ can be derived as:
\begin{equation}
\begin{aligned}
\frac{\partial N}{\partial\hat{y}_{\textbf{k}}} = \frac{-2(r_{\textbf{k}}-\overline{r}_{p})}{S-1}.
\nonumber
\end{aligned}
\end{equation}
So the single block loss $\mathcal{L}_{p}=|\pi _{p}-1|$ on the exact form of its gradient:
\begin{equation}
\begin{aligned}
\frac{\partial(N/D)}{\partial\hat{y}_{\textbf{k}}} &= \frac{D(\partial N/\partial\hat{y}_{\textbf{k}}) - N(\partial D/\partial\hat{y}_{\textbf{k}})}{D^2}\\
&=\mathrm{sgn}(\pi_{p}-1) \cdot \frac{\frac{-2(r_{\textbf{k}}-\overline{r}_{p})}{S-1}(k\overline{y}_{p}+\epsilon) - k\frac{Var_{p}(r)}{S}}{(k\overline{y}_{p}+\epsilon)^2}.
\nonumber
\end{aligned}
\end{equation}

\end{proof}

Denote the standard deviation of the residuals on block $p$ as $\sigma_r$. From the structure of the gradient formula, we can see that the modulus length of the gradient satisfies the following relation:
\begin{equation}
\begin{aligned}
||\frac{\partial\mathcal{L}_{p}}{\partial\hat{y}_{\textbf{k}}}|| \in \Theta\left(\frac{\sigma_r}{k\overline{y}_{p}+\epsilon}\right),
\end{aligned}
\end{equation}
where $\Theta$ provide a asymptotic tight bound of a function and $||\cdot||$ is defined as the norm.

Since in the Poisson scenario the residual variance $\sigma_r^2 \approx k\overline{y}_p$, the relation can be further derived as follows:
\begin{equation}
\begin{aligned}
||\frac{\partial\mathcal{L}_{p}}{\partial\hat{y}_{\textbf{k}}}|| \in \Theta\left(\frac{\sqrt{k\overline{y}_p}}{k\overline{y}_{p}+\epsilon}\right) \approx \Theta\left(\frac{1}{\sqrt{k\overline{y}_p}}\right).
\end{aligned}
\end{equation}

The relation $\Theta((\overline{y}_p)^{-1/2})$ describes gradient adaptivity, which holds for $k\overline{y}_p \gg \varepsilon$. In the low count region, the gradient is upper bounded by $\epsilon$, avoiding gradient explosion.







\subsubsection{Interpretation as GMM}
\begin{theorem}
PVMC-Loss can be interpreted as an implementation of the Generalized Moment Matching method (GMM) \cite{hansen1982large}.
\end{theorem}

\begin{proof}
GMM is a method for parameter estimation by matching a set of moment conditions that are theoretically expected to be zero. For our problem, we can define the following moment conditions:
\begin{equation}
\begin{aligned}
m_1(\theta) = \mathbb{E}(x - f_{\theta}(x)) = 0,
\nonumber
\end{aligned}
\end{equation}
\begin{equation}
\begin{aligned}
m_2(\theta) = \mathbb{E}(Var_{p}(x - f_{\theta}(x)) - (k \cdot Mean_{p}(f_{\theta}(x)) + \epsilon)) = 0,
\nonumber
\end{aligned}
\end{equation}
where $m_1(\cdot)$ is the first moment condition, $m_2(\cdot)$ is the second moment condition and $\theta$ is defined as the set of parameters of the network.

In our total training objective, the L1 loss term mainly drives the network to satisfy the first-order moment condition, while the PVMC-Loss term can be viewed as an L1-norm form penalty term built around the second moment condition. Since the gradient of the network $\nabla_\theta f_\theta(x)\neq0$ holds almost everywhere in the parameter space, the Jacobian $D_\theta\mathbf{m}(\theta) =(m_1,m_2)^\top$ of the moment vector $\mathbf{m}(\theta) $ is generally expected to have full rank under typical training conditions. This satisfies the identification condition of Hansen et al. \cite{hansen1982large} and helps ensure the consistency of the GMM estimator. Unlike Poisson NLL, which aims to match the entire probability distribution, low-order moment based GMM strategies are computationally simpler and rely less on the exact morphological assumptions of the full distribution, which generally makes them more robust in the face of slight mismatches between model and data.

\end{proof}

\subsection{PC-UNet}
\begin{table}[!t]
\caption{Comparision Experiments.The best results are in \textbf{bold} and the second best are \underline{underlined}.}
\begin{center}
\begin{tabular}{cccc}
\toprule
Method& PSNR& SSIM& TIME\\
\midrule
 GANLC \cite{liu2023solving}& 32.70& 0.9616&0.0208 $\pm$ 0.0001\\
 CoreDiff \cite{gao2023corediff}& \textbf{37.83}& \underline{0.9795}&0.1544 $\pm$ 0.0024\\
 U-Net \cite{ronneberger2015u}& 35.99& 0.9699&\textbf{0.0062} $\pm$ \textbf{0.0010}\\
 SwinUnet \cite{cao2022swin}& 37.10& 0.9750&0.0280 $\pm$ 0.0030\\
 VM-Unet \cite{ruan2024vm}& 37.20& 0.9760&0.0210 $\pm$ 0.0025\\
 CSWin-Unet \cite{liu2025cswin}& 37.25& 0.9770&0.0320 $\pm$ 0.0035\\
 \textbf{PC-UNet (ours)}& \underline{37.68}& \textbf{0.9809}&\underline{0.0078 $\pm$ 0.0011}\\
 \bottomrule
 \end{tabular}
\label{tab:comparision}
\end{center}
\end{table}

U-Net features a symmetric encoder-decoder structure, where the encoder extracts hierarchical image features through convolutions and downsampling, and the decoder restores spatial resolution via upsampling. Key-hop connections link encoder and decoder feature maps at matching scales, alleviating the vanishing gradient problem and enhancing high-frequency detail transfer. Our PC-UNet incorporates the proposed PVMC-Loss.

PC-UNet is trained end-to-end by optimizing a composite loss function that aims to simultaneously guarantee the fidelity and physical consistency of the generated images. The total training objective $\mathcal{L}_{\text{total}}$ is defined as follows:
\begin{equation}
\begin{aligned}
\mathcal{L}_{\text{total}} = \mathcal{L}_{\text{L1}} + \lambda \cdot \mathcal{L}_{\text{PVMC}},
\end{aligned}
\end{equation}
where $\lambda$ is a scalar hyperparameter that balances the two optimization goals of data fidelity and physical consistency, and $\mathcal{L}_{\text{L1}} = ||\hat{y} - y||_1$ is the standard of L1 loss. As a Data Fidelity Term, it drives the network output $\hat{y}$ to approximate the gold standard image $y$ at the voxel level, ensuring the overall similarity of the image content. It is worth noting that within the $\mathcal{L}_{PVMC}$ term, the network's own output $\hat{y}$ is used as an approximation of the true signal mean. This bootstrapping approach is a common and effective strategy in self-consistent optimization problems.


\section{Experiments}

\subsection{Dataset}
We use subjects 1 to 60 from Bern-Inselspital-2022 in the UDPET Challenge 2024 dataset \cite{xue2022cross}. The 1\%-2\% low-dose images serve as noisy inputs, with corresponding full-dose images as targets for paired training. Of the dataset, 40 pairs are for training, and 20 pairs for testing.
\subsection{Comparision Experiments}
\subsubsection{Setting}
We conduct experiments on a system with 8 NVIDIA RTX A6000 GPUs. All baseline models use the hyperparameters from their original papers. Both PC-UNet and U-Net employ a 4-layer U-Net architecture with encoder feature maps: [64, 128, 256, 512]. The network processes single-channel grayscale images to single-channel outputs. We use learnable transposed convolutions for upsampling in the decoder. Each convolution block is followed by Batch Normalization, and a 0.1 dropout rate is applied to prevent overfitting.

\begin{figure}[t]
    \centering
    \includegraphics[width=0.90\linewidth]{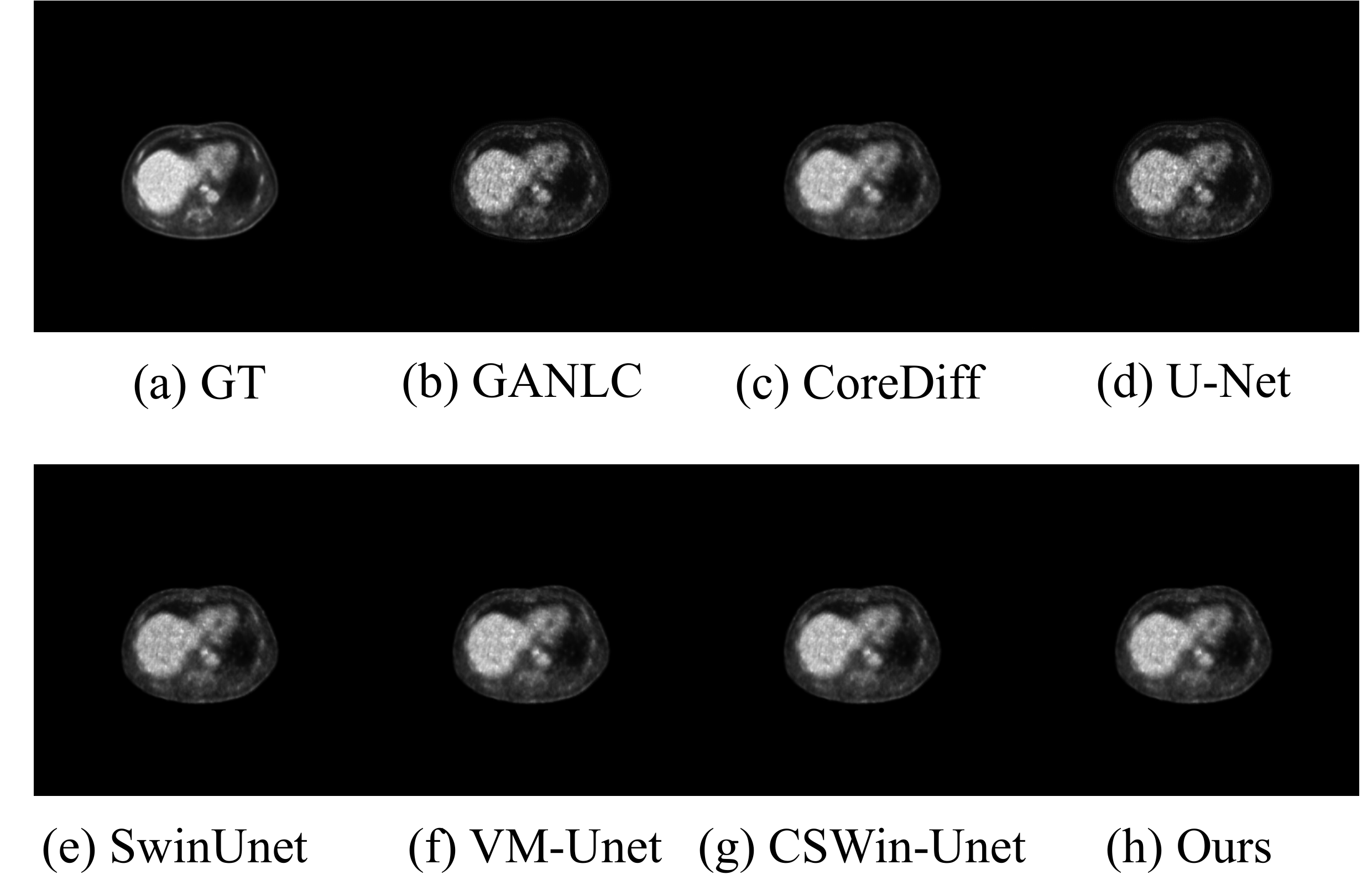}
    \caption{
The denoising results of different methods.
    }
    \label{fig:compar}
\end{figure}

Models are trained with the Adam optimizer for up to 1500 epochs, using early stopping based on validation metrics. Training uses a batch size of 16 and an initial learning rate of 1e-4, reduced by a scheduler to at least 1e-7. We fix the random seed at 3407 for reproducibility. Patches are set to $16^2$, $k$ starts at 0.8, and $\lambda$ is set to 1e-5.

\subsubsection{Evaluation Metrics}
We use PSNR and SSIM \cite{wang2004image} as evaluation metrics and introduce a TIME metric to showcase the lightweight U-Net backbone, measuring the model's reasoning time for an image. We report the average and variance of TIME across three experiments.

\subsubsection{Results}
The comparative experiments in \cref{tab:comparision} show our model's optimal PSNR and SSIM in the U-Net architecture. The denoising results in \cref{fig:compar} indicate that PC-UNet closely approaches optimality and leads in SSIM. While slightly slower than the standard U-Net, our model outperforms DDPM and GAN in time, narrowing the PSNR gap.

\begin{figure}[t]
    \centering
    \subfloat[]{
        \includegraphics[width=0.30\linewidth]{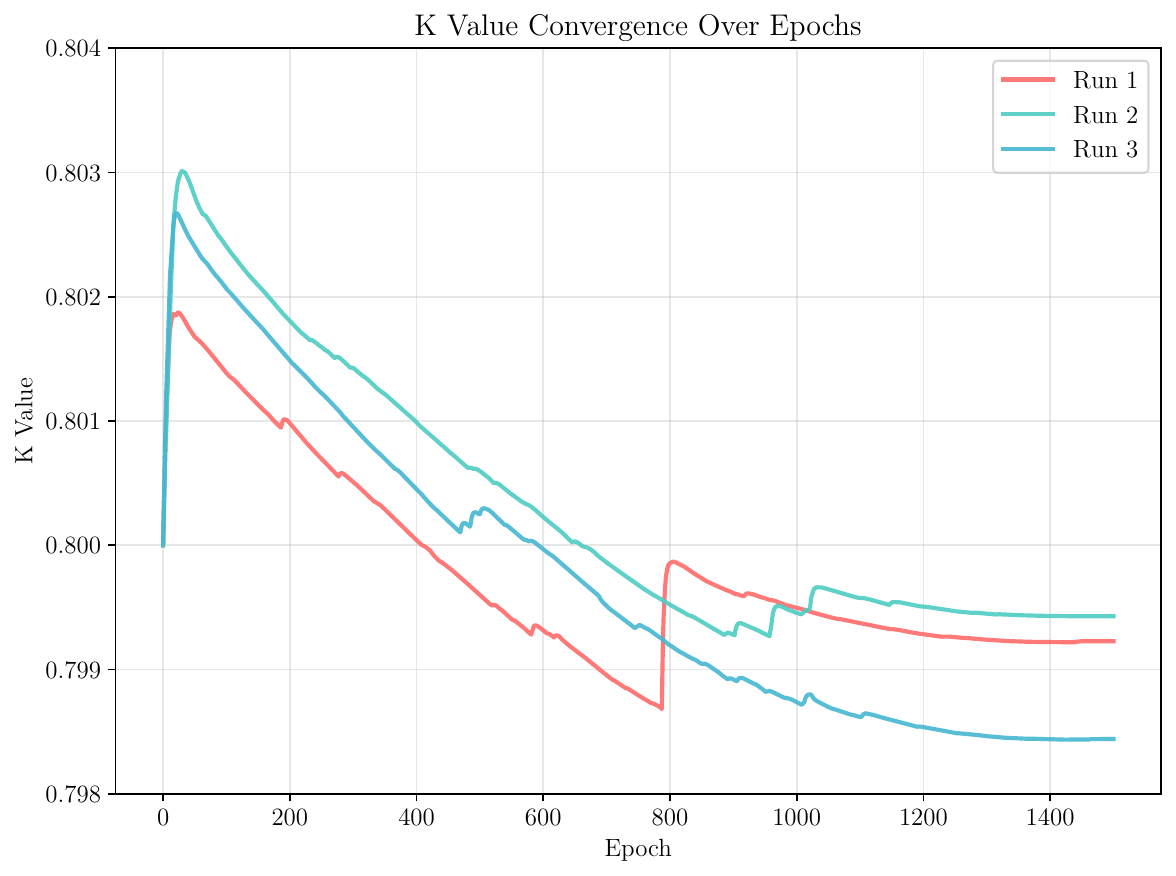}
        \label{fig:k}
    }
    \subfloat[]{
        \includegraphics[width=0.30\linewidth]{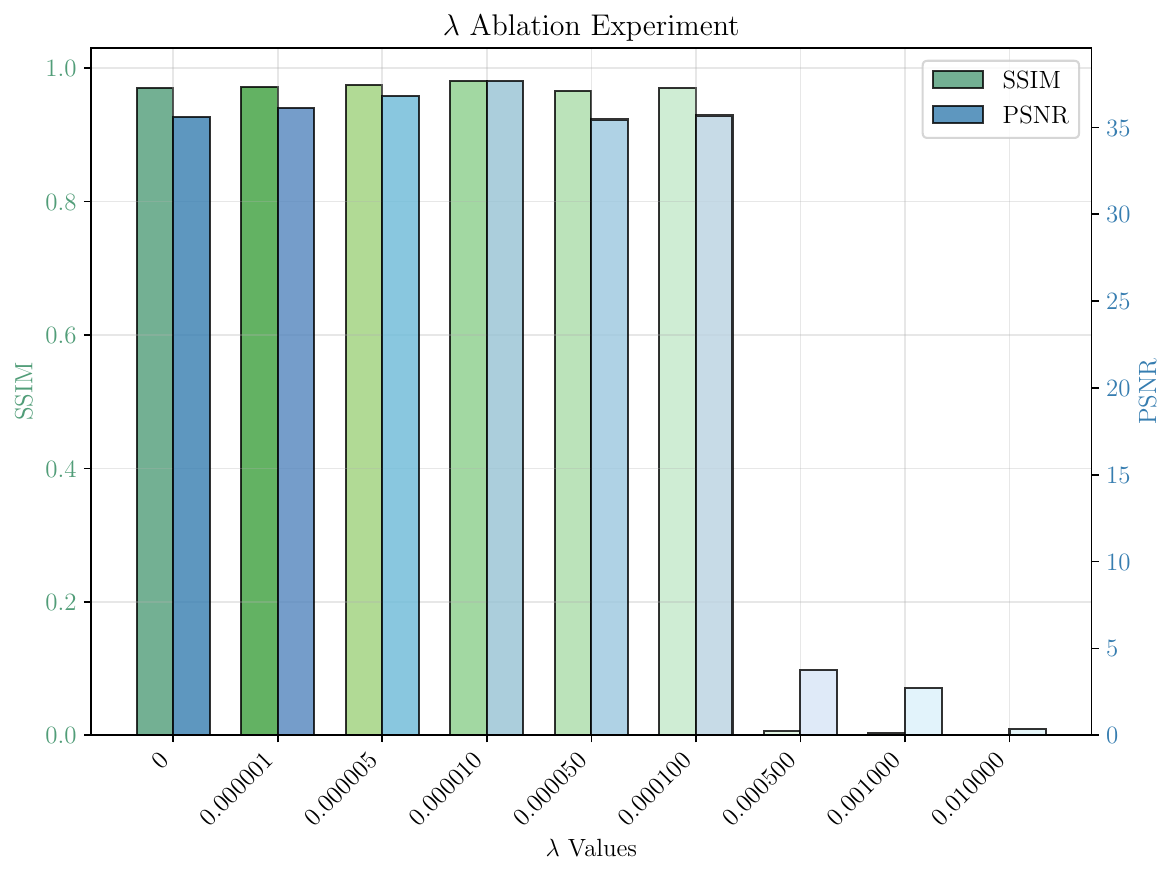}
        \label{fig:lambda}
    }
    \subfloat[]{
        \includegraphics[width=0.30\linewidth]{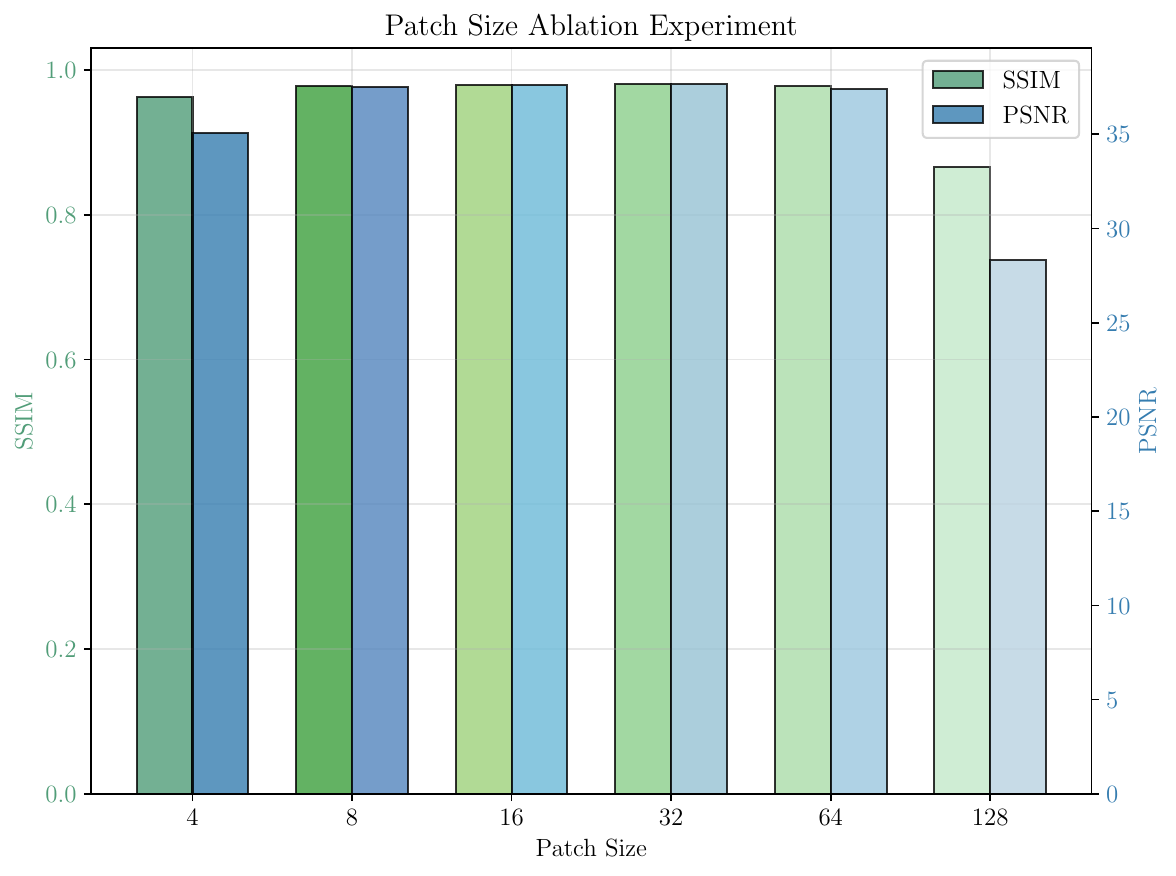}
        \label{fig:patch}
    }
    \caption{(a) $k$ converges as epochs increase; (b) ablation over 9 values of $\lambda$; (c) ablation over 6 patch sizes.}
    \label{fig:ablations}
\end{figure}

\subsubsection{Analysis of parameter $k$}
To verify the proposed parameter $k$'s physical validity, we compare it with actual physical parameters. Since the physical parameter $k$ cannot be obtained online, we divide the dataset from the same device into three equal parts and train each separately. The $k$ value change with training rounds is shown in \cref{fig:k}. Results show parameters $k$ from different datasets converge within 0.001, suggesting it may represent the real physical parameter. This confirms our neural network-based parameter $k$ retains physical properties.

\subsubsection{Analysis of Hyperparameter $\lambda$}
The parameters $\lambda$ are key to the PVMC-loss. We fix patches to $16^2$, choosing $\lambda$ from $\{$0, 1e-2, 1e-3, 5e-4, 1e-4, 5e-5, 1e-5, 5e-6, 1e-6$\}$. Results in \cref{fig:lambda} show accuracy initially increases, then decreases. When $\lambda$ is 0, PC-UNet becomes standard U-Net, proving PVMC-loss effectiveness. High $\lambda$ decreases accuracy, revealing that excessive physical consistency can reduce performance.

\subsubsection{Analysis of Hyperparameter Patches}

We set $\lambda=$1e-5 and choose patch values in \cref{equ:k} from the set $\{$$4^2$, $8^2$, $16^2$, $32^2$, $64^2$, $128^2$$\}$. Results are shown in \cref{fig:patch}. For patches $8^2$, $16^2$, $32^2$, or $64^2$, PSNR and SSIM values are similar, showing our method's robustness. With patches $4^2$, PSNR and SSIM slightly drop due to statistical instability overshadowing improvements in physical model fidelity, as reliable means and variances in small patches are harder to compute. At patches $128^2$, SSIM and PSNR decline sharply because physical constraints fail, and random sampling adds uncertainty, hindering effective learning. We conclude that the patches hyperparameter is broadly robust and doesn't need minor adjustments in practice.

\section{Conclusion}
Experiments show that our PC-UNet significantly improves PET denoising. We provide a theoretical analysis of PVMC-Loss, demonstrating its asymptotic unbiasedness and gradient adaptability, and its connection to the GMM framework. However, PVMC-Loss derivation assumes uniform local radioactivity distribution, suitable for backgrounds or homogeneous tissues. This may be inaccurate at sharp boundaries between tumors and normal tissues. Future research could explore better methods for obtaining $k$.

\bibliographystyle{IEEEtran}
\bibliography{references}

\begin{thebibliography}{10}
\providecommand{\url}[1]{#1}
\csname url@samestyle\endcsname
\providecommand{\newblock}{\relax}
\providecommand{\bibinfo}[2]{#2}
\providecommand{\BIBentrySTDinterwordspacing}{\spaceskip=0pt\relax}
\providecommand{\BIBentryALTinterwordstretchfactor}{4}
\providecommand{\BIBentryALTinterwordspacing}{\spaceskip=\fontdimen2\font plus
\BIBentryALTinterwordstretchfactor\fontdimen3\font minus \fontdimen4\font\relax}
\providecommand{\BIBforeignlanguage}[2]{{%
\expandafter\ifx\csname l@#1\endcsname\relax
\typeout{** WARNING: IEEEtran.bst: No hyphenation pattern has been}%
\typeout{** loaded for the language `#1'. Using the pattern for}%
\typeout{** the default language instead.}%
\else
\language=\csname l@#1\endcsname
\fi
#2}}
\providecommand{\BIBdecl}{\relax}
\BIBdecl

\bibitem{emnlp/XieZCHC23}
Y.~Xie, Z.~Zhu, X.~Cheng, Z.~Huang, and D.~Chen, ``Syntax matters: Towards spoken language understanding via syntax-aware attention,'' in \emph{{EMNLP}}, 2023, pp. 11\,858--11\,864.

\bibitem{mm/XieZCCH24}
Y.~Xie, Z.~Zhu, X.~Chen, Z.~Chen, and Z.~Huang, ``Moba: Mixture of bi-directional adapter for multi-modal sarcasm detection,'' in \emph{ACM MM}.\hskip 1em plus 0.5em minus 0.4em\relax {ACM}, 2024, pp. 4264--4272.

\bibitem{wang2024novel}
J.~Wang, Z.~Deng, T.~Lin, W.~Li, and S.~Ling, ``A novel prompt tuning for graph transformers: Tailoring prompts to graph topologies,'' in \emph{Proceedings of the 30th ACM SIGKDD Conference on Knowledge Discovery and Data Mining}, 2024, pp. 3116--3127.

\bibitem{wang2024beyond}
J.~Wang, Z.~Deng, T.~Lin, W.~Li, S.~Ling, and J.~Lin, ``Beyond direct relationships: Exploring multi-order label pair dependencies for knowledge distillation,'' in \emph{Proceedings of the 32nd ACM International Conference on Multimedia}, 2024, pp. 8527--8535.

\bibitem{hussain2022modern}
S.~Hussain, I.~Mubeen, N.~Ullah, S.~S. U.~D. Shah, B.~A. Khan, M.~Zahoor, R.~Ullah, F.~A. Khan, and M.~A. Sultan, ``Modern diagnostic imaging technique applications and risk factors in the medical field: a review,'' \emph{BioMed research international}, vol. 2022, no.~1, p. 5164970, 2022.

\bibitem{bottomley1982nmr}
P.~A. Bottomley, ``Nmr imaging techniques and applications: a review,'' \emph{Review of Scientific Instruments}, vol.~53, no.~9, pp. 1319--1337, 1982.

\bibitem{pain2022deep}
C.~D. Pain, G.~F. Egan, and Z.~Chen, ``Deep learning-based image reconstruction and post-processing methods in positron emission tomography for low-dose imaging and resolution enhancement,'' \emph{European Journal of Nuclear Medicine and Molecular Imaging}, vol.~49, no.~9, pp. 3098--3118, 2022.

\bibitem{seyyedi2024deep}
N.~Seyyedi, A.~Ghafari, N.~Seyyedi, and P.~Sheikhzadeh, ``Deep learning-based techniques for estimating high-quality full-dose positron emission tomography images from low-dose scans: a systematic review,'' \emph{BMC Medical Imaging}, vol.~24, no.~1, p. 238, 2024.

\bibitem{shi2019novel}
L.~Shi, J.~A. Onofrey, E.~M. Revilla, T.~Toyonaga, D.~Menard, J.~Ankrah, R.~E. Carson, C.~Liu, and Y.~Lu, ``A novel loss function incorporating imaging acquisition physics for pet attenuation map generation using deep learning,'' in \emph{MICCAI}, 2019, pp. 723--731.

\bibitem{boellaard2015fdg}
R.~Boellaard, R.~Delgado-Bolton, W.~J. Oyen, F.~Giammarile, K.~Tatsch, W.~Eschner, F.~J. Verzijlbergen, S.~F. Barrington, L.~C. Pike, W.~A. Weber \emph{et~al.}, ``Fdg pet/ct: Eanm procedure guidelines for tumour imaging: version 2.0,'' \emph{European journal of nuclear medicine and molecular imaging}, vol.~42, no.~2, pp. 328--354, 2015.

\bibitem{akita202518f}
R.~Akita, K.~Takauchi, M.~Ishibashi, S.~Kondo, S.~Ono, K.~Yokomachi, Y.~Ochi, M.~Kiguchi, H.~Mitani, Y.~Nakamura \emph{et~al.}, ``18f-fdg dose reduction using deep learning-based pet reconstruction,'' \emph{EJNMMI research}, vol.~15, no.~1, p.~78, 2025.

\bibitem{yan2016method}
J.~Yan, J.~Schaefferkoetter, M.~Conti, and D.~Townsend, ``A method to assess image quality for low-dose pet: analysis of snr, cnr, bias and image noise,'' \emph{Cancer Imaging}, vol.~16, no.~1, p.~26, 2016.

\bibitem{shepp2007maximum}
L.~A. Shepp and Y.~Vardi, ``Maximum likelihood reconstruction for emission tomography,'' \emph{TIP}, vol.~1, no.~2, pp. 113--122, 2007.

\bibitem{hu2023comparative}
Y.~Hu, D.~Lv, S.~Jian, L.~Lang, C.~Cui, M.~Liang, L.~Song, S.~Li, and Z.~Wu, ``Comparative study of the quantitative accuracy of oncological pet imaging based on deep learning methods,'' \emph{Quantitative Imaging in Medicine and Surgery}, vol.~13, no.~6, p. 3760, 2023.

\bibitem{liu2024cross}
X.~Liu, S.~V. Eslahi, T.~Marin, A.~Tiss, Y.~Chemli, Y.~Huang, K.~A. Johnson, G.~El~Fakhri, and J.~Ouyang, ``Cross noise level pet denoising with continuous adversarial domain generalization,'' \emph{Physics in Medicine \& Biology}, vol.~69, no.~8, p. 085001, 2024.

\bibitem{hansen1982large}
L.~P. Hansen, ``Large sample properties of generalized method of moments estimators,'' \emph{Econometrica: Journal of the econometric society}, pp. 1029--1054, 1982.

\bibitem{zaidi2007scatter}
H.~Zaidi and M.-L. Montandon, ``Scatter compensation techniques in pet,'' \emph{PET clinics}, vol.~2, no.~2, pp. 219--234, 2007.

\bibitem{lehtinen2018noise2noise}
J.~Lehtinen, J.~Munkberg, J.~Hasselgren, S.~Laine, T.~Karras, M.~Aittala, and T.~Aila, ``Noise2noise: Learning image restoration without clean data,'' \emph{arXiv}, 2018.

\bibitem{batson2019noise2self}
J.~Batson and L.~Royer, ``Noise2self: Blind denoising by self-supervision,'' in \emph{ICML}, 2019, pp. 524--533.

\bibitem{liu2023solving}
W.~Liu and H.~Ding, ``Solving low-dose ct reconstruction via gan with local coherence,'' in \emph{MICCAI}, 2023, pp. 524--534.

\bibitem{gao2023corediff}
Q.~Gao, Z.~Li, J.~Zhang, Y.~Zhang, and H.~Shan, ``Corediff: Contextual error-modulated generalized diffusion model for low-dose ct denoising and generalization,'' \emph{TIP}, vol.~43, no.~2, pp. 745--759, 2023.

\bibitem{ronneberger2015u}
O.~Ronneberger, P.~Fischer, and T.~Brox, ``U-net: Convolutional networks for biomedical image segmentation,'' in \emph{MICCAI}, 2015, pp. 234--241.

\bibitem{cao2022swin}
H.~Cao, Y.~Wang, J.~Chen, D.~Jiang, X.~Zhang, Q.~Tian, and M.~Wang, ``Swin-unet: Unet-like pure transformer for medical image segmentation,'' in \emph{ECCV}, 2022, pp. 205--218.

\bibitem{ruan2024vm}
J.~Ruan, J.~Li, and S.~Xiang, ``Vm-unet: Vision mamba unet for medical image segmentation,'' \emph{arXiv}, 2024.

\bibitem{liu2025cswin}
X.~Liu, P.~Gao, T.~Yu, F.~Wang, and R.-Y. Yuan, ``Cswin-unet: Transformer unet with cross-shaped windows for medical image segmentation,'' \emph{Information Fusion}, vol. 113, p. 102634, 2025.

\bibitem{xue2022cross}
S.~Xue, R.~Guo, K.~P. Bohn, J.~Matzke, M.~Viscione, I.~Alberts, H.~Meng, C.~Sun, M.~Zhang, M.~Zhang \emph{et~al.}, ``A cross-scanner and cross-tracer deep learning method for the recovery of standard-dose imaging quality from low-dose pet,'' \emph{European journal of nuclear medicine and molecular imaging}, pp. 1--14, 2022.

\bibitem{wang2004image}
Z.~Wang, A.~C. Bovik, H.~R. Sheikh, and E.~P. Simoncelli, ``Image quality assessment: from error visibility to structural similarity,'' \emph{TIP}, vol.~13, no.~4, pp. 600--612, 2004.

\end{thebibliography}

\end{document}